\documentclass[shortpaper]{clv3}

\usepackage{hyperref}
\usepackage{xcolor}
\definecolor{darkblue}{rgb}{0, 0, 0.5}
\hypersetup{colorlinks=true,citecolor=darkblue, linkcolor=darkblue, urlcolor=darkblue}

\usepackage{amssymb}
\usepackage{amsthm,bm}
\usepackage{amsmath}
\usepackage{booktabs}
\usepackage{multirow}
\usepackage{graphicx}
\usepackage{paralist}
\usepackage{enumitem}
\usepackage{arydshln}

\begin{document}
\issue{51}{3}{2025}

\runningtitle{Unsupervised Chunking with Hierarchical RNN}
\runningauthor{Wu et al.}

\pageonefooter{Action editor: Deyi Xiong. Submission received: 11 September 2023; revised version received: 17 September 2024; accepted for publication: 23 October 2024.\\}

\title{The Emergence of Chunking Structures with Hierarchical RNN}

\author{Zijun Wu$^{1}$, Anup Anand Deshmukh$^{2}$\thanks{Work partially done as a co-op intern at the University of Alberta/Amii.}, Yongkang Wu$^{3}$, \\Jimmy Lin$^{2}$, Lili Mou$^{1,4}$}

\affilblock{
    \affil{University of Alberta, Department of Computing Science, Alberta Machine
Intelligence Institute (Amii)\\ zijun4@ualberta.ca, doublepower.mou@gmail.com}
    \affil{University of Waterloo, David R. Cheriton, School of Computer Science\\ aa2deshmukh@uwaterloo.ca, jimmylin@uwaterloo.ca}
    \affil{Huawei Poisson Lab\\ wuyongkang7@huawei.com}
    \affil{Canada CIFAR AI Chair}
}

\maketitle

\begin{abstract}
In Natural Language Processing (NLP), predicting linguistic structures, such as parsing and chunking, has mostly relied on manual annotations of syntactic structures. This paper introduces an unsupervised approach to chunking, a syntactic task that involves grouping words in a non-hierarchical manner. We present a Hierarchical Recurrent Neural Network (HRNN) designed to model word-to-chunk and chunk-to-sentence compositions. Our approach involves a two-stage training process: pretraining with an unsupervised parser and finetuning on downstream NLP tasks. Experiments on multiple datasets reveal a notable improvement of unsupervised chunking performance in both pretraining and finetuning stages. Interestingly, we observe that the emergence of the chunking structure is transient during the neural model's downstream-task training. This study contributes to the advancement of unsupervised syntactic structure discovery and opens avenues for further research in linguistic theory.\footnote{This paper is a substantially extended version of an EMNLP Findings paper~\cite{deshmukh-etal-2021-unsupervised-chunking}. The main extension is a novel approach that improves unsupervised chunking performance in a downstream NLP task, which is an interesting result by itself. Part of the text is reused with permission from our previous work~\cite{deshmukh-etal-2021-unsupervised-chunking}, available at \url{https://aclanthology.org/2021.findings-emnlp.307.pdf}, licensed under \href{https://creativecommons.org/licenses/by/4.0/}{CC BY 4.0}. Our code and results are available at \url{https://github.com/MANGA-UOFA/UCHRNN}}
\end{abstract}

\section{Introduction}\label{Sec:intro}
Understanding the linguistic structure of language, such as parsing and chunking, is an important research topic in natural language processing (NLP). While most previous studies~\cite{kudo2001chunking, zhang2002text, pradhan-etal-2004-shallow} use supervised machine learning methods to predict linguistic structures and can achieve high performance, they rely heavily on manual annotations of syntactic structures like Treebanks~\citep{marcus-etal-1993-building}. As a result, there has been a growing interest in unsupervised linguistic structure discovery in recent years~\citep{Kim2020Are, shen2018neural, shen2018ordered}, which is important to NLP research because it sheds light on linguistic theories and can potentially benefit low-resource languages. Previous work on unsupervised syntactic structure discovery has mainly focused on unsupervised constituency parsing~\cite{Kim2020Are, kim-etal-2019-compound, shen2018neural, shen2018ordered} and dependency parsing~\citep{klein-manning-2004-corpus, gillenwater-etal-2010-sparsity, shen-etal-2021-structformer}.

In this work, we address unsupervised chunking, another meaningful task for discovering syntactic structure. Unlike parsing that induces tree structures from a sentence, chunking aims to group consecutive words of a sentence in a non-hierarchical fashion, and each chunk can be intuitively thought of as a phrase~\citep{sang2000introduction, clark-2001-unsupervised}. In fact, unsupervised chunking has wide applications in real-world scenarios, as understanding text fundamentally requires finding spans like noun phrases and verb phrases. It would benefit other NLP tasks, such as keyword extraction~\citep{firoozeh_nazarenko_alizon_daille_2020}, named entity recognition~\citep{sato-etal-2017-segment}, open information extraction~\citep{niklaus-etal-2018-survey, pmlr-v162-borgeaud22a, izacard2022few}, and logical reasoning~\citep{wu2023weakly}.

We propose a hierarchical recurrent neural network (HRNN) to accomplish the chunking task. Our HRNN is designed to explicitly model the word-to-chunk composition by a lower-level RNN, and chunk-to-sentence composition by an upper-level RNN. We further design a trainable chunking gate that switches between lower word-level RNN and upper phrase-level RNN, which is also used for chunk prediction.

We propose a two-stage training framework for HRNN without using linguistic annotations, namely, pretraining with an unsupervised parser and finetuning with a downstream NLP task. In the first stage of pretraining, we adopt the recent advances of unsupervised parsing~\citep{Kim2020Are, kim-etal-2019-compound}, and propose a maximal left-branching heuristic to induce sensible (albeit noisy and imperfect) chunk labels from an unsupervised parser. In the second stage, we finetune HRNN chunking by feeding its upper-level RNN representations to a downstream-task network. Our intuition is that HRNN serves as the bottleneck, and the downstream-task finetuning can force the HRNN to learn more meaningful latent structures.

We conducted experiments on the CoNLL-2000 dataset~\citep{sang2000introduction} and the English Web Treebank~\citep{annwebtreebank}. Results show that the unsupervised parser-pretrained HRNN significantly improves the best-performing unsupervised baseline, with a considerable margin of up to $6$ percentage points in terms of the phrase F1 score.
We then finetuned the pretrained HRNN model with three downstream text generation tasks: summarization, paraphrasing, and translation. Compared with our pretrained model, it achieves a chunking performance improvement of up to $6$ percentage points in downstream datasets, and $2$ percentage points on CoNLL-2000. The results suggest that our method not only bridges the gap between supervised and unsupervised chunking methods but also shows the generalizability across different downstream tasks and datasets.

In our experiment, we also find an intriguing linguistic phenomenon during the finetuning stage: the neural model's emergence of linguistic structure is transient. That is, although the downstream performance consistently improves, the chunking performance improves significantly but only in the first few thousand steps. Then, it starts to decrease and eventually drops to the initial level of the finetuning stage or even lower, suggesting that linguistic structures are a convenient vehicle for a downstream task when the training begins and the model capacity is small. However, the neural network tends to discard such linguistic structures to achieve higher performance in the downstream task as the training proceeds.

In summary, our main contributions are as follows:
\begin{itemize}
    \item We address the task of unsupervised chunking, and propose a hierarchical recurrent neural network (HRNN) that can explicitly model the word-to-chunk and chunk-to-sentence compositions.
    \item We propose a two-stage training framework for HRNN, largely outperforming previous unsupervised chunking methods.
    \item We observe the neural model’s emergence of chunking structure is transient, which may inspire further study of linguistic theory.
\end{itemize}

\section{Approach}
In this section, we first introduce the proposed Hierarchical RNN (HRNN) model. Then, we discuss the process of pretraining HRNN with the induced labels from a state-of-the-art unsupervised parser. Finally, we propose to finetune the model with a downstream task to improve chunking performance.

\subsection{Hierarchical RNN}
\label{sec: Hierarchical RNN}

\begin{figure}[!t]
\includegraphics[width=\textwidth]{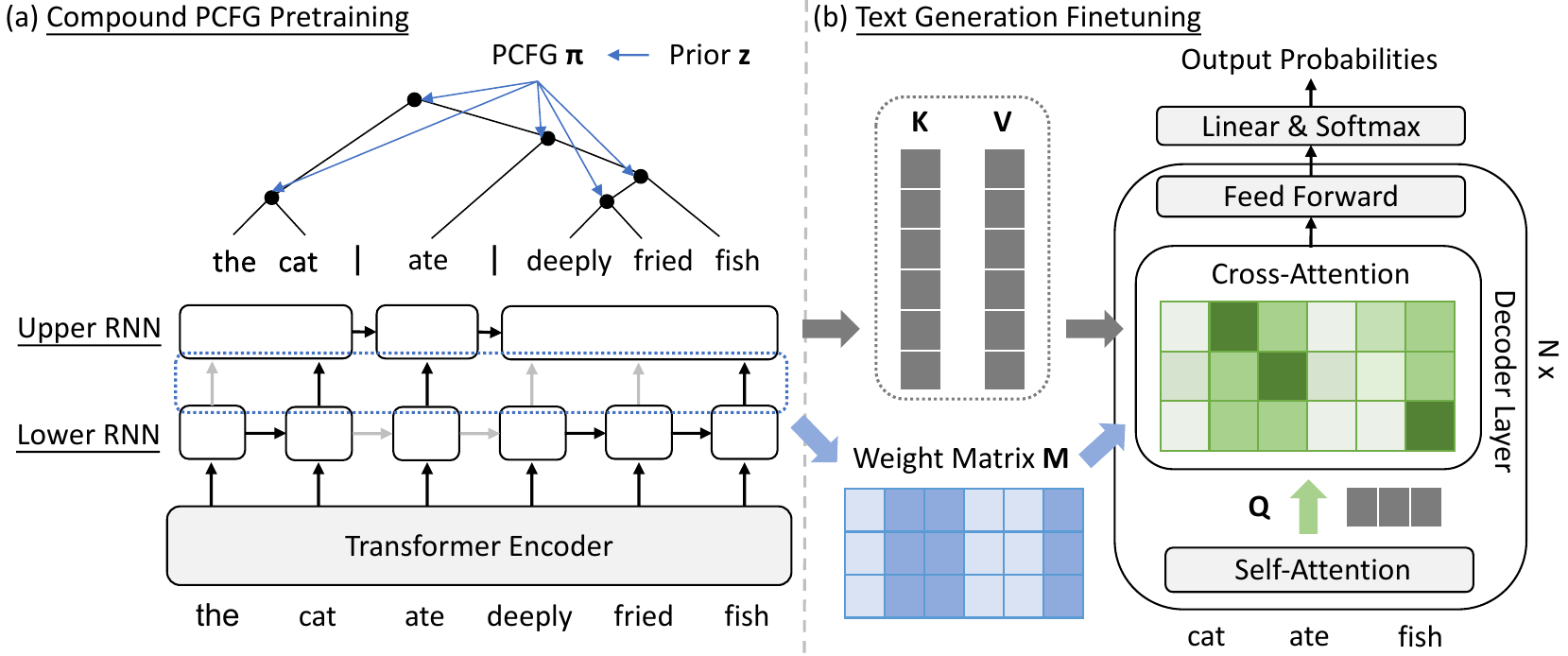}
\caption{An overview of our chunking induction method. (a) Pretraining HRNN using the chunk labels induced from the Compound PCFG parser. (b) HRNN is finetuned with text generation, specifically a summarization task in this example. A weight matrix is created from the switching gate's values and then inserted into the Transformer's encoder-decoder attention modules.}
\label{fig: intro to method}
\end{figure}

We model the chunking task in the BI schema~\citep{ramshaw-marcus-1995-text}, where “B” denotes the beginning of a chunk, and “I” denotes the inside of a chunk. We observe that encoder-only RNNs or Transformers may not be suitable for the chunking task because they lack autoregressiveness, which means that the prediction at a time step is unaware of previously predicted chunks. Feeding predicted chunk labels like a sequence-to-sequence model is also inadequate because a BI label only provides one-bit information and does not offer meaningful autoregressive information. 

To this end, we design a hierarchical RNN to model the autoregressiveness of predicted chunks. Our HRNN contains a lower-level RNN operating at the word level and an upper-level RNN operating at the chunk level. We design a gating mechanism that switches between the two RNNs in a soft manner, also serving as the chunk label. The bottom-left part of Figure~\ref{fig: intro to method} shows the structure of HRNN.

Given a sentence with words $\mathrm x^{(1)}, \cdots, \mathrm x^{(n)}$, we first apply a pretrained Transformer encoder~\citep{devlin-etal-2019-bert} to obtain the contextual representations of the words that help our model to understand the previous and future context of the sentence. We denote the contextual embeddings of these words by $\bm x^{(1)}, \cdots, \bm x^{(n)}$.  For a step $t$, we predict a switching gate $m^{(t)}\in (0,1)$ as the chunking decision.\footnote{Here, $m^{(t)} = 1$ corresponds to ``B,'' i.e.,  a new chunk,  and $m^{(t)}= 0$ corresponds to ``I,'' i.e., inside of a chunk.}
\begin{align}
    &m_\text{logit}^{(t)} = W \bigl[\underline{\bm{h}}^{(t-1)}; \overline{\bm{h}}^{(t-1)}; \bm{x}^{(t)} \bigr] \label{eqn: raw switching gate} \\
    &m^{(t)} = \sigma (m_\text{logit}^{(t)})
\end{align}
where $\underline{\bm{h}}^{(t-1)}$ is the hidden state of the lower RNN and $\overline{\bm{h}}^{(t-1)}$ is that of the upper RNN. The semicolon represents vector concatenation, and $\sigma$ represents the sigmoid function.

The switching gate is also used to control the information flow between the lower-level word RNN and the upper-level chunk RNN. In this way, it provides meaningful autoregressive information that enables HRNN to be aware of previously detected chunks.

Suppose our model predicts the $t$th word as the beginning of a chunk, which essentially “cuts” the sequence into two parts at this step. The lower RNN and upper RNN are updated by
\begin{align}
    &\underline{\bm{h}}^{(t)}_\text{cut} = \underline{f}\bigl(\bm{x}^{(t)}, \underline{\bm{h}}^\text{(sos)} \bigr) \label{eqn:lower cut update formula}\\
    &\overline{\bm{h}}^{(t)}_\text{cut} = \overline{f}\bigl(\underline{\bm{h}}^{(t-1)}, \overline{\bm{h}}^{(t-1)} \bigr) \label{eqn:upper cut update formula}
\end{align}
where $\underline{f}$ and $\overline{f}$ are the transition functions of the lower and upper RNNs, respectively, both employing the hyperbolic tangent (tanh) as their activation function.

In Equation~(\ref{eqn:lower cut update formula}), the lower RNN ignores its previous hidden state and starts anew from a learnable initial state $\underline{\bm{h}}^\text{(sos)}$, because a new chunk is predicted at this step. On the other hand, in Equation~(\ref{eqn:upper cut update formula}), the upper RNN picks the newly formed chunk $\underline{\bm{h}}^{(t-1)}$ captured by the lower RNN, and fuses it with the previous upper RNN's state $\overline{\bm{h}}^{(t-1)}$. 

Suppose our model predicts that the $t$th word is not the beginning of a chunk, i.e., “no cut” occurs at this step. The RNNs are then updated by
\begin{align}
    &\underline{\bm{h}}^{(t)}_\text{nocut} = \underline{f}\bigl(\bm{x}^{(t)}, \underline{\bm{h}}^{(t-1)} \bigr) \label{eqn:lower no cut update formula}\\
    &\overline{\bm{h}}^{(t)}_\text{nocut} = \overline{\bm{h}}^{(t-1)} \label{eqn:upper no cut update formula}
\end{align}
In this scenario, the lower RNN updates its hidden state with the input $\bm{x}^{(t)}$ in the same manner as a typical RNN, whereas the upper RNN remains idle because no chunk is formed.

The “cut” and “no cut” cases can be unified by
\begin{align}
    &\overline{\bm{h}}^{(t)} = m^{(t)} \overline{\bm{h}}^{(t)}_\text{cut} + (1-m^{(t)}) \overline{\bm{h}}^{(t)}_\text{nocut} \label{eqn:upper soft update formula}\\
    &\underline{\bm{h}}^{(t)} = m^{(t)} \underline{\bm{h}}^{(t)}_\text{cut} + (1-m^{(t)}) \underline{\bm{h}}^{(t)}_\text{nocut} \label{eqn:lower soft update formula}
\end{align}
In our work, we keep $m^{(t)}$ as a real number, which is essentially a gating mechanism that fuses “cut” and “no cut” in a soft manner. This is because chunking, by its nature, may be ambiguous, and our soft gating mechanism can better handle such ambiguity.

Our proposed HRNN supports end-to-end learning, which can be weakly supervised by downstream tasks. However, a randomly initialized HRNN may have difficulty modeling the desired chunking structure and suffer from the cold-start problem. We thus propose to pretrain HRNN by heuristically induced chunking labels from an unsupervised parser, which will be discussed in Section~\ref{sec: Warming-up Hierarchical RNN}.

\subsection{Pretraining HRNN by Unsupervised Parsing}\label{sec: Warming-up Hierarchical RNN}
We propose to induce chunk labels from an unsupervised parser. The intuition is that the chunking structure can be considered a flattened parse tree, thus sharing some commonalities with the parsing structure. Our approach is able to take advantage of recent advances in unsupervised parsing~\citep{Kim2020Are, kim-etal-2019-compound}.

Specifically, we adopt a state-of-the-art unsupervised parser, Compound Probabilistic Context-Free Grammar (PCFG; \citealp{kim-etal-2019-compound}). A PCFG is a tuple $\mathcal{G} = (\mathcal{S}, \mathcal{N}, \mathcal{P}, \Sigma, \mathcal{R})$, where $\mathcal{S}$ is a start symbol; $\mathcal{N}$, $\mathcal{P}$, and $\Sigma$ are finite sets of nonterminal, preterminal, and terminal symbols, respectively. $\mathcal{R}$ is a finite set of rules taking one of the following forms:
\begin{align}
    &\mathrm{S} \rightarrow \mathrm{A}    &\mathrm{A} \in \mathcal{N} \\
    &\mathrm{A} \rightarrow \mathrm{B}\ \mathrm{C} & \mathrm{B}, \mathrm{C} \in \mathcal{N} \cup \mathcal{P}\\
    &\mathrm{T} \rightarrow \mathrm{w}   &\mathrm{T} \in \mathcal{P}, \mathrm{w} \in \Sigma
\end{align}
where $\mathrm{S} \rightarrow \mathrm{A}$ is the start of a sentence and $\mathrm{T} \rightarrow \mathrm{w}$ indicates the generation of a word. $\mathrm{A} \rightarrow \mathrm{B}\ \mathrm{C}$ models the bifurcations of a binary constituency tree, where a constituent node is not explicitly associated with a type (e.g., noun phrase) in our setting.

In addition, the model maintains a continuous random vector at the sentence level, which serves as the prior of PCFG. To train Compound PCFG, the maximum likelihood of the text is utilized, and a Viterbi-like algorithm marginalizes the PCFG. Amortized variational inference is employed to handle the continuous distribution from the sentence-level vector. We refer readers to~\citet{kim-etal-2019-compound} for details.

We aim to induce chunk labels from the state-of-the-art unsupervised parser, Compound PCFG. Given a sentence, we obtain its parse tree by applying the Viterbi-like CYK algorithm~\citep{YOUNGER1967189,10.5555/555733} to Compound PCFG. Then, we propose a simple yet effective heuristic that extracts maximal left-branching subtrees as chunks. It is known that the English language has a strong bias toward right-branching structures~\citep{williams-etal-2018-latent, li-etal-2019-imitation}. We observe, on the other hand, that a left-branching structure often indicates words that are closely related. Here, a left-branching subtree means that the words are grouped in the form of $(( \cdots ( (\mathrm{x}_{i} \mathrm{x}_{i+1})\mathrm x_{i+2} )\cdots )\mathrm{x}_{i+n-1})$. A left-branching subtree for words $\mathrm{x}_i \cdots \mathrm{x}_{i+n-1}$ is maximal if neither $\mathrm{x}_{i-1}\mathrm x_i \cdots \mathrm{x}_{i+n-1}$ nor $\mathrm{x}_i \cdots \mathrm{x}_{i+n-1} \mathrm x_{i+n}$ is left-branching. We extract all maximal left-branching subtrees as chunks. Our heuristic can provide unambiguous chunk labels for any sentence with any parse tree, as demonstrated by the following theorem.

\begin{theorem}\label{thm:main}
Given any binary parse tree, every word will belong to one and only one chunk by the maximal left-branching heuristic.
\end{theorem}

\begin{proof}

[Existence] A single word itself is a left-branching subtree, which belongs to some maximal left-branching subtree.

[Uniqueness] We will show that two different maximal left-branching subtrees $\textrm s_1$ and $\textrm s_2$ cannot
overlap. Assume by way of contradiction that there exists a word $\textrm x_i$ in both $\textrm s_1$ and $\textrm s_2$. Then, $\textrm s_1$ must be a substructure of $\textrm s_2$ or vice versa; otherwise, the paths,  $\textrm{root} \text{–} \textrm{s}_1 \text{–} \textrm{x}_i$ and $\textrm{root} \text{–} \textrm{s}_2 \text{–} \textrm{x}_i$, violate the acyclic nature of a tree. But $\textrm s_1$ being a subtree of $\textrm s_2$ (or vice versa) contradicts the maximality of $\textrm s_1$ and $\textrm s_2$.
\end{proof}

The theorem shows that our maximal left-branching heuristic can unambiguously assign chunk labels based on a binary parse tree of any sentence.

We induce the labels by our heuristic and pretrain HRNN on those labels. While the heuristic is simple and may not be perfect, it still achieves satisfactory chunking results. Let $T$ be the length of a sentence. Our pretraining objective is to minimize the following loss function for a sample:
\begin{align}\label{eqn: pretrain loss}
    \mathcal{L}_\text{pretrain} = - \sum_{t=1}^{T} \left[\widetilde{y}_{t} \cdot \log m^{(t)}
    + (1-\widetilde{y}_{t}) \cdot \log(1 - m^{(t)}) \right]
\end{align}
where $\widetilde{y}_{t}$ is the induced chunk label for the $t$th step, and $m^{(t)}$ is the chunking decision predicted by HRNN. As a machine learning model, the HRNN is able to smooth out the noise from the heuristic and yield more meaningful chunks, shown in Section~\ref{sec: Main Results}.

\subsection{Finetuning Hierarchical RNN with Downstream Tasks}
\label{sec: Finetuning Hierarchical RNN with Downstream Tasks}
We propose to finetune the HRNN in a weakly supervised manner to learn better chunking structures from a downstream task. Our intuition is that accomplishing a downstream NLP task requires understanding meaningful semantic units/chunks of a sentence~\citep{wu2023weakly}, and that optimizing the downstream task may benefit the chunk prediction.

Specifically, we consider text generation tasks, namely, summarization~\cite{rush-etal-2015-neural, liu-etal-2022-learning, liu2019fine}, paraphrasing~\cite{liu-etal-2020-unsupervised, li-etal-2019-decomposable}, and translation~\cite{bojar-etal-2014-findings, liu-etal-2020-multilingual-denoising}. Although the HRNN may be finetuned with classification tasks as well, we found in our preliminary experiments that classification tasks yield marginal improvement, probably because the classification training signal is too sparse (one label per sample) in comparison to generation tasks. Therefore, we only consider finetuning HRNN with text generation tasks in this work.

We propose a decoder network, whose attention is based on the upper-level RNN (chunk representations). This is accomplished by performing conventional attention over the RNN's states but reweighing it by whether a chunk is predicted at the step.
Since the HRNN's gate is differentiable, we can train it end to end with the downstream tasks.

We first argue that the chunk representation is indeed given by the upper-level RNN where a ``cut'' is predicted.
Suppose we have a chunk $\mathrm x^{(i)} \cdots \mathrm x^{(i+k)}$, i.e., HRNN predicts two consecutive ``cuts'' (end of a chunk) at the time steps $i-1$ and $i+k$. According to our design (Section~\ref{sec: Hierarchical RNN}), the lower-level RNN initializes its hidden state with $\underline{\bm{h}}^{\text{(sos)}}$ when $t = i-1$, and updates its hidden state $\underline{\bm{h}}^{(t)}$ with the input words $\mathrm x^{(i)} \cdots \mathrm x^{(i+k)}$. In the meantime, the upper-level RNN is idle for time steps $i$ to $i+k-1$, but picks up the $\underline{\bm{h}}^{(i+k)}$ at step $i+k$. This follows
a conventional RNN, demonstrating that the chunk representation is given by the upper-level RNN at a ``cut'' step.

Now consider the scenarios with ``soft'' chunks, since we keep the chunking decision $m^{(t)}$ as a real number in our HRNN. When the ``cut'' strength is high ($m^{(t)}$ close to 1), the step is more likely to be a chunk, and we would like to attend to this step more from the decoder. Conversely, if the  ``cut'' strength is low, it should be less attended. This can be achieved by applying conventional encoder--decoder attention to every step, but reweighing it by the ``cut'' strength. For each attention head, we form a query vector $\bm{q}^{(j)}$ for the $j$th decoder step and a key vector $\bm{k}^{(i)}$ for the $i$th encoder step. Here, the key vector $\bm{k}^{(i)}$ is given by the upper-level RNN's state $\overline{\bm h}^{(i)}$.
We reweigh the attention by first computing an unnormalized measure:
\begin{align} \label{eqn: weighed attention}
\widetilde\alpha \Bigl( \bm{q}^{(j)}, \bm{k}^{(i)}, m_\text{logit}^{(i)} \Bigr) 
= \operatorname{exp}{\Bigl( \frac{(\bm{q}^{(i)})^\top \bm{k}^{(j)}}{\sqrt{d}}+\gamma m_\text{logit}^{(i)} \Bigr)}
\end{align}
where $\gamma$ is a coefficient hyperparameter. Then, we normalize it to attention probabilities as 
\begin{align} 
\alpha^{(j, i)}=\frac{\widetilde{\alpha} \Bigl(\bm{q}^{(j)}, \bm{k}^{(i)},m_\text{logit}^{(i)} \Bigr)}{\sum_{i'}\widetilde\alpha\Bigl(\bm{q}^{(j)}, \bm{k}^{(i')}, m_\text{logit}^{(i')} \Bigr)}
\end{align}

Empirically, we find that an end-to-end reweighted attention head tends to ignore $m_\text{logit}^{(i)}$ and make it close to 0. Consequently, the gating probability $m^{(i)}$ will be close to 0.5 unanimously, and the HRNN will be in a situation of neither ``cut'' nor ``no cut.'' We alleviate this by proposing an auxiliary training objective that pushes $m^{(i)}$ to either~0 or~1. We introduce a chunking strength hyperparameter $\kappa\%$, which is the desired chunk--token ratio. We then define the loss at the $i$th encoder step as
\begin{align}
L^{(i)}=
\begin{cases}
    (m^{(i)}-1)^2 , & \text{if $m^{(i)}$ is among the top-$\kappa$\% gate values,} \\
    (m^{(i)}-0)^2, & \text{otherwise.}
\end{cases}
\label{eqn: aux loss}
\end{align}
A detailed analysis of this hyperparameter is presented in Section~\ref{sec: analysis of finetuning}, and the value of $\kappa\%$ is decided by validation performance.

The overall auxiliary loss is $L_\text{aux}=\sum_i L^{(i)}$, and is combined with the end-to-end downstream task by $L_\text{task} + \eta L_\text{aux}$, where $\eta$ is a hyperparameter balancing the two training objectives. In this way, we are able to learn a meaningful chunking gate in the downtream NLP task.

\section{Experiments}

\begin{table*}[!t]
\caption{Details on the dataset and implementation. Symbol $\eta$ represents the balance between two training objectives, while $\gamma$ is the hyper-parameter for the reweight operation, as defined in Equation~\ref{eqn: weighed attention}. The term top-$\kappa$ denotes the chunking strength heuristics.}
\label{tab: dataset and implementation}
\resizebox{\linewidth}{!}{
\begin{tabular}{l | l | c c c | c c c c c c c} 
\toprule
 Task & Dataset & $\#$Train & $\#$Dev & $\#$Test & Batch size & Learning rate & $\eta$ & $\gamma$ & top-$\kappa$\% \\
\midrule
\multirow{2}{*}{Chunking} & CoNLL-2000 & 8K & 1K & 2K & 32 & 5e-5 & -- & -- & -- \\
& English Web Treebank & -- & 1.9K & 1K & -- & -- & -- & -- & -- \\
\midrule
\multirow{3}{*}{Downstream} & Gigaword & 3.8M & 2K &2K &  32 & 5e-5 & 0.1 & 0.1 & 50$\%$ \\
& MNLI-Gen & 131K & 3.5K & 3.5K & 32 & 4e-5 & 0.1 & 0.1 & 60$\%$ \\
& WMT-14 (En-De) & 4.5M & 3K & 2.7K & 32 & 1e-5 & 0.1 & 0.1 & 50$\%$ \\
\bottomrule
\end{tabular}}
\end{table*}

\subsection{Datasets and Metrics}
Table~\ref{tab: dataset and implementation} summarizes the datasets used in our study. We adopted a widely used chunking dataset, CoNLL-2000~\citep{sang2000introduction}, to evaluate the general chunking performance of a model. CoNLL-2000 contains 8K training, 1K validation, and 2K test samples. Each sample is labeled with the BIO schema, where ``B'' indicates the beginning of a chunk, ``I'' indicates inside a chunk, and ``O'' indicates outside a chunk (mainly punctuation). We used the BI schema and ignore the ``O'' tokens.

In this work, however, we addressed chunking in the unsupervised setting, and thus we did not use the groundtruth chunk labels in the training set, but only the training sentences. Groundtruth chunk labels are used for validation\footnote{We will discuss the use of labeled validation dataset in Section~\ref{section: Effect of the Size of Validation Data on Unsupervised Performance}.} and test purposes. 

Further, we used the English Web Treebank~\citep{annwebtreebank} to evaluate the general chunking performance on the online review domain, following \citet{deshmukh-etal-2021-unsupervised-chunking}. The dataset does not contain linguistically annotated chunk labels, and therefore, we resorted to the NLTK toolkit~\cite{bird-loper-2004-nltk} to obtain pseudo-groundtruth chunk labels.

We finetuned the HRNN model with several downstream text generation tasks, whose datasets are introduced as follows.
\begin{itemize}
    \item Gigaword~\citep{rush-etal-2015-neural}: A widely used article--headline dataset for text summarization with about 3.8M training, 20K validation, and 2K test samples. We randomly picked 2K samples (same size as the test set) for more efficient validation. 
    \item MNLI-Gen: MNLI~\citep{williams-etal-2018-broad} is a massive dataset for the natural language inference task. We used the entailment subset, generating the entailed hypothesis based on the premise. Since the test set of MNLI is private, we adopted the matched section of the original validation set as our test set, while the mismatched section served as our validation set. The resulting entailment generation dataset contains 131K training, 3.5K validation, and another 3.5K test samples.
    \item WMT-14~\citep{bojar-etal-2014-findings}: A multilingual dataset used for machine translation. We used the English-to-German subset, having nearly 4.5M training, 3K validation, and 2.7K test samples.
\end{itemize}
In addition to the general chunking performance on the CoNLL-2000 and English Web Treebank corpora, we are interested in in-domain chunking performance after finetuning with a downstream task, whose test set, unfortunately, does not have human annotation of chunking labels. Therefore, we also used NLTK~\citep{bird-loper-2004-nltk} to produce pseudo-groundtruth labels for the source sentences of the downstream datasets.

Regarding the evaluation metrics, we adopted the standard phrase F1 score and tagging accuracy. In our implementation, they are realized by the CoNLL-2000 evaluation script~\citep{sang2000introduction}. 

\subsection{Implementation Details}
We utilized an encoder-only BERT model and two encoder--decoder models, BART and mBART, in our experiments. For the BERT model~\cite{devlin-etal-2019-bert}, we adopted its base version that has $12$ Transformer layers with $768$ hidden dimensions, and its total parameter number is $110$ million. For the BART model~\cite{lewis-etal-2020-bart}, we also used its base version, which has $6$ encoder layers and another $6$ decoder layers. It has 768 dimensions and 139 million parameters in total. The mBART model~\cite{liu-etal-2020-multilingual-denoising} has 610 million parameters, both the encoder and decoder having $12$ layers and  $1024$ dimensions. We only used the encoder module of BART or mBART for the chunking pretraining, which is then combined with the respective pretrained decoder for finetuning on downstream tasks. Our HRNN has the same hidden dimension as the underlying Transformer models.

We used the Adam optimizer~\citep{kingma2014adam} for both pretraining and finetuning. The optimizer hyperparameters were set as $\beta_1=0.9$ and$\beta_2=0.999$. We applied a learning rate warm-up for the first $10$ percent of total steps, followed by a linear decay. The detailed hyperparameters for both stages can also be found in Table~\ref{tab: dataset and implementation}. 

The pretraining phase generally concludes within three epochs, whereas the finetuning process requires less than one. Although one-epoch finetuning may appear unreasonable at first glance, we observed that Transformers are very powerful models, which may accomplish a downstream task while bypassing syntactic information. To let meaningful syntax emerge in the downstream task, we have to consider early stopping and only very few iterations are needed. Details will be presented in Section~\ref{sec: analysis of finetuning}.

\subsection{Main Results}
\label{sec: Main Results}

\begin{table*}[!t]
\centering
\caption{Main results on the linguistic annotated data.}
\label{tab: main result ling}
\resizebox{0.8\linewidth}{!}{
\begin{tabular}{l l  l l  l l } 
\toprule
  \multirow{2}{*}{\#} &\multirow{2}{*}{Method} & \multicolumn{2}{c}{\textbf{CoNLL-2000}} & \multicolumn{2}{c}{\textbf{English Web Treebank}} \\
 & & Phrase F1 & Tag Acc. & Phrase F1 & Tag Acc. \\
 \midrule
 \multicolumn{6}{l}{\textbf{Supervised Methods}} \\
 1 & NLTK-tagger-chunker & 83.71 & 89.51 & -- & -- \\
 2 & Supervised HMM & 87.68 & 93.99 & 98.62 & 99.44\\
 \midrule
 \multicolumn{6}{l}{\textbf{Unsupervised Methods}} \\
 3 & PMI Chunker & 35.64 & 64.50 & 32.28 & 65.34 \\
 4 & Baum–Welch HMM & 25.04 & 58.93 & 24.17 & 58.02 \\
 5 & LM Chunker & 42.05 & 68.74 & 31.23 & 62.55\\
 6 & Compound PCFG Chunker & 62.89 & 81.64 & 58.17 & 79.33\\
 7 & Llama2-7B Prompting & 48.12 & 73.28 &  43.79 & 64.34   \\
 \midrule
 \multicolumn{6}{l}{\textbf{Our Pretrain Methods}} \\
 8 & HRNN w/ BERT & 68.12 & 83.90 & 64.32 & 83.25\\
 9 & HRNN w/ BART & 68.57 & 84.04 & 69.71 & 83.08 \\
 10 & HRNN w/ mBART & 68.65 & 84.42 & 67.53 & 82.18 \\
 \midrule
 \multicolumn{6}{l}{\textbf{Our Finetune Methods}} \\
 11 & HRNN w/ BART (Gigaword) & \textbf{70.83} & \textbf{85.16} & \textbf{72.05} & \textbf{84.16} \\
 12 & HRNN w/ BART (MNLI-Gen) & 69.68 & 84.72 & 70.47 & 83.12\\
 13 & HRNN w/ mBART (WMT-14) & 69.77 & 84.69 & 69.57 & 83.06 \\
 
\bottomrule
\end{tabular}}
\end{table*}

\begin{table*}[!t]
\caption{Main results of the chunking performance on the downstream datasets. The numbers with \textsuperscript{\textdagger} indicate that the model is tested on the same datasets as the finetuning tasks.}
\label{tab: main result downstream}
\resizebox{\linewidth}{!}{
\begin{tabular}{l l  l l   l l  l l } 
\toprule
  \multirow{2}{*}{\#} &\multirow{2}{*}{Method} & \multicolumn{2}{c}{\textbf{Gigaword}} & \multicolumn{2}{c}{\textbf{MNLI-Gen}} & \multicolumn{2}{c}{\textbf{WMT-14}} \\
 & & Phrase F1 & Tag Acc. & Phrase F1 & Tag Acc. & Phrase F1 & Tag Acc. \\
 \midrule
 \multicolumn{8}{l}{\textbf{Supervised Methods}} \\
 1 & Supervised HMM & 87.54 & 94.37 & 86.81 & 93.40 & 86.49 & 93.50 \\
 \midrule
 \multicolumn{8}{l}{\textbf{Unsupervised Methods}} \\
 2 & PMI Chunker  & 37.44 & 65.18 & 44.37 & 67.14 & 45.70 & 67.20\\
 3 & Baum–Welch HMM  & 38.35 & 54.75 & 29.87 & 32.06 & 34.87 & 45.00\\
 4 & LM Chunker & 29.31 & 51.34 & 25.85 & 49.57 & 32.84 & 50.72 \\
 5 & Compound PCFG Chunker  & 25.69 & 51.70 & 32.96 & 55.73 & 52.16 & 73.73 \\
 6 & Llama2-7B Prompting & 34.38 & 65.72 & 38.59 &  60.91 & 43.31 & 66.56 \\
 \midrule
 \multicolumn{8}{l}{\textbf{Our Pretrain Methods}} \\
 7 & HRNN w/ BERT & 64.20 & 82.61 & 66.38 & 81.60 & 67.27 & 82.70 \\
 8& HRNN w/ BART &  64.60 & 82.24 & 68.81 & 82.04 & 72.09 & 84.87 \\
 9 & HRNN w/ mBART & 67.08 & 83.49 & 67.25 & 81.35 & 71.24 & 84.55 \\
 \midrule
 \multicolumn{8}{l}{\textbf{Our Finetune Methods}} \\
 10 & HRNN w/ BART (Gigaword) & \textbf{70.66}\textsuperscript{\textdagger} & \textbf{85.15}\textsuperscript{\textdagger} & \textbf{73.96} & \textbf{84.93} & \textbf{75.87}  & \textbf{86.66} \\
 11 & HRNN w/ BART (MNLI-Gen) & 67.15 & 83.36  & 70.79\textsuperscript{\textdagger} & 83.11\textsuperscript{\textdagger} & 73.26 & 85.43 \\
 12 & HRNN w/ mBART (WMT-14) & 67.91 & 83.58 & 69.93 & 82.75 & 71.85\textsuperscript{\textdagger} & 84.69\textsuperscript{\textdagger} \\
 
\bottomrule
\end{tabular}}
\end{table*}

Table~\ref{tab: main result ling} presents the main results of general chunking performance. In the experiment, we followed~\citep{deshmukh-etal-2021-unsupervised-chunking} and trained all models on CoNLL-2000; testing was performed on CoNLL-2000 as well as the Englisht Web Treebank.

For comparison of inducing chunks from unsupervised parsing, we include Compound PCFG (discussed in Section~\ref{sec: Warming-up Hierarchical RNN}) and another recent unsupervised parser based on the features of a pretrained language model proposed by \citet{Kim2020Are}. The latter, called an LM Chunker, thresholds the BERT~\citep{devlin-etal-2019-bert} similarity of consecutive words for chunking. We observe from Lines~5 and~6 in Table~\ref{tab: main result ling} that the LM-based unsupervised chunker is worse than the Compound PCFG in phrase F1 (42.05 vs.~62.89 on CoNLL-2000; 31.23 vs.~58.17 on English Web Treebank) and tag accuracy (68.74 vs.~81.64 on CoNLL-2000; 62.55 vs.~79.33 on English Web Treebank). Therefore, we use the chunking labels induced by Compound PCFG to pretrain our HRNN model in all experiments. 

We include several traditional unsupervised chunking methods as baselines: a pointwise mutual information (PMI) chunker~\citep{van-de-cruys-2011-two} cuts two consecutive words if their PMI score is below a threshold; Baum--Welch HMM~\citep{baggenstoss2001modified} applies expectation--maximization to a hidden Markov model (HMM, \citealp{18626}). These methods perform significantly worse than the recent advances in unsupervised syntactic structure discovery. 

However, we notice that both Compound PCFG and LM chunkers (Lines 4--5 in Table~\ref{tab: main result downstream}) show a significant decrease in performance when transferring to other datasets. For example, when the two chunkers are transferred from CoNLL-2000 to Gigaword, their tag accuracy is dropped by 25 percent and 37 percent, respectively. 
By contrast, the PMI chunker has a 7 percent drop, while the Baum--Welch HMM even has a 1 percent increase (Lines 3--4). The results indicate that traditional unsupervised chunking methods are more stable, although it achieves lower performance on CoNLL-2000. Overall, unsupervised chunkers may suffer from the problem of poor transferability in addition to less satisfactory chunking performance in general.

We experimented with a baseline by prompting a large language model (LLM). Specifically, we prompted Llama2-7B~\citep{touvron2023llama} to generate the chunked sentences (separated by ``|'').\footnote{Our prompt is:
\texttt{<s>[INST] <<SYS>>\textbackslash n You’re a helpful assistant who is an expert in linguistics. The user will ask you to separate a given sentence into phrase chunks using ‘|'. In NLP, chunking refers to the process of segmenting and labeling multi-token sequences, such as phrases or sub-sentences. For example, for the sentence ‘The cat ate deeply fried fish’, the desired answer is ‘The cat | ate | deeply fried fish’. The user wants you to directly give the answer without explanation. \textbackslash n<</SYS>>\textbackslash n \textbackslash n Please separate the sentence '\%s' into phrases. [/INST]}} Notice that our prompt has been reasonably well engineered with role-playing and exemplar techniques, and that an LLM may not be fully unsupervised because it is likely to be pretrained with linguistically annotated corpora (such as PTB). That being said, its performance (Line~7 in Table~\ref{tab: main result ling} and Line~6 in Table~\ref{tab: main result downstream}) is less satisfactory and lower than Compound PCFG. The results suggest that, although LLMs are good at generating fluent text, it remains a challenge to discover linguistic structures with LLMs.

Regarding our method, we pretrained HRNN with several Transformer encoders: 
\begin{enumerate}[label=\alph*)]
    \item BERT~\cite{devlin-etal-2019-bert}, an encoder-only Transformer model used in our preliminary conference paper~\cite{deshmukh-etal-2021-unsupervised-chunking}.
    \item BART~\cite{lewis-etal-2020-bart}, chosen due to its pretrained decoder, which allows HRNN to be finetuned on Gigaword and MNLI-Gen datasets for generation tasks. 
    \item mBART~\cite{liu-etal-2020-multilingual-denoising}, a multilingual variant of BART, which is needed for the translation task.
\end{enumerate}

We see from Lines 8--10 in Table~\ref{tab: main result ling} that all the pretrained HRNN variants achieve an improvement of more than $5$ percentage points in phrase F1 based on the Compound PCFG chunker on the CoNLL-2000 dataset. BART-based and mBART-based HRNN is better than the BERT-based variant not only on the standard CoNLL-2000 dataset but also in transferring its learned chunking knowledge to three downstream datasets. The large margins between the Compound PCFG chunker and our pretrained HRNN model imply that HRNN can indeed smooth out the noise of heuristics while capturing meaningful chunking patterns.\footnote{It is understandable that a machine learning system may smooth out the noise of training labels. For example, our HRNN achieves an F1 score of 96.67\% in a supervised setting; if we randomly flip 40\% of the labels, the F1 only decreases to 91.52\%. Such a smoothing effect is also observed in other domains, such as text generation~\cite{NEURIPS2020_7a677bb4, wen2024ebbs} and syntactic parsing~\cite{li-etal-2019-imitation, shayegh2024ensemble}.} More importantly, our method drastically improves the transferability of unsupervised chunkers (Lines 7--9 vs.~Lines 2--5) in Table~\ref{tab: main result downstream}.

We then finetuned the HRNN on some downstream task for syntactic structure discovery. 
The results are shown in Table~\ref{tab: main result downstream}. 
We observe an improved chunking performance across all datasets, including Gigaword, MNLI-Gen, and WMT-14. As shown in Lines 11--12 (Table~\ref{tab: main result downstream}),  finetuning HRNN with MNLI-Gen and WMT-14 leads to a 1--2 percent increase in phrase F1 on the in-domain datasets (70.79 vs.~68.81; 71.85 vs.~71.24). Futhermore, as seen in Table~\ref{tab: main result ling}, it leads to improvement in the general CoNLL-2000 corpus (69.68 vs.~68.57; 69.77 vs.~68.65) and English Web Treebank corpus (70.47 vs.~69.71; 69.57 vs.~67.53).

Notably, finetuning with the summarization task yields the most significant improvements, as indicated in Line 11 in Table~\ref{tab: main result ling} and Line 10 in Table~\ref{tab: main result downstream}. It achieves nearly an increase of $6$ percentage points in phrase F1 on the in-domain Gigaword dataset (70.66 vs.~64.60), as well as an increase of $5$ and $4$ percentage points on the out-of-domain MNLI-Gen and WMT-14 datasets (73.96 vs.~68.81; 75.87 vs.~72.09). Furthermore, it brings another 2.2-point improvement to the CoNLL-2000 dataset. The superiority of chunk induction from the summarization task may be due to the inherent connection between summarization and chunking. This is in line with a summarization method known as phrasal extractive summarization~\citep{muresan-etal-2001-combining, xu-zhao-2022-jointly}, where the salient information (often key phrases) is extracted to generate summary text. Future work may further analyze the connection between linguistic structures and real-world tasks, perhaps with linguistic probing methods.

It should be emphasized that our findings are different from the those in the multi-task transfer learning literature. For example, \citet{phang2018sentence} suggest that pretraining a model with intermediate tasks can help the downstream task, but both are in the supervised regime; more interestingly, \citet{phang-etal-2020-english} demonstrate that such transferability holds cross different languages in a multilingual model. 
However, our chunker is never trained with linguistic annotations, but the meaningful linguistic structure emerges with the supervision signal of the downstream task only.

\subsection{Detailed Analyses} \label{sec: analysis}
In this subsection, we conduct comprehensive experiments to verify our proposed approach. We will first show an analysis of the proposed heuristic for chunk label induction from unsupervised parsers for pretraining HRNN. Then, we will show an ablation study on the proposed HRNN structure. Finally, we will analyze the effectiveness of finetuning HRNN on downstream text generation tasks.

\begin{table*}[t]
    \caption{Analysis of the heuristics for inducing chunking labels from Compound PCFG.}
    \label{tab: chunking heuristics}
    \centering
    \resizebox{0.56\linewidth}{!}{
    \begin{tabular}{l c c}
    \toprule
    \textbf{Chunking Heuristics} & \textbf{Phrase F1} & \textbf{Tag Acc.} \\
    \midrule
    1-word \& 2-word chunks & 55.72 & 75.14 \\
    Maximal right branching & 40.83 & 69.28 \\
    Maximal left branching & \textbf{62.89} & \textbf{81.64} \\
    \bottomrule
    \end{tabular}}
\end{table*}

\begin{table*}[t]
    \caption{Ablation study of the HRNN model.}
    \label{tab: HRNN ablation}
    \centering
    \resizebox{0.6\linewidth}{!}{
    \begin{tabular}{c l c c}
    \toprule
    \textbf{\#} & \textbf{Method} & \textbf{Phrase F1} &  \textbf{Tag Acc.} \\
    \midrule
    1 & Compound PCFG chunker & 62.89 & 81.64 \\
    \hdashline
    2 & $\longrightarrow$ 1-layer RNN & 63.40 & 81.70 \\
    3 & $\longrightarrow$ 2-layer RNN & 64.59 & 82.32 \\
    4 & $\longrightarrow$ HRNN only & 65.01 & 82.22 \\
    \hdashline
    5 & $\longrightarrow$ BERT tagger & 66.67 & 83.34 \\
    6 & $\longrightarrow$ BERT+1-layer RNN & 67.19 & 83.86 \\
    7 & $\longrightarrow$ BERT+2-layer RNN & 66.53 & 83.34 \\
    8 & $\longrightarrow$ BERT+HRNN (hard) & 67.90 & 83.80 \\
    9 & $\longrightarrow$ BERT+HRNN & \textbf{68.12} & \textbf{83.90} \\
    \bottomrule
    \end{tabular}}
\end{table*}

\subsubsection{Analysis of the Left-Branching Chunking Heuristic}
\label{sec: Left-branching Chunking Heuristic}
We provide a detailed analysis of our maximal left-branching chunking heuristic, where we used the CoNLL-2000 dataset as a testbed because we pretrained all HRNN variants on CoNLL-2000. Without loss of generality, we conducted this analysis with  BERT-based HRNN. 

Table~\ref{tab: chunking heuristics} shows a comparison between different heuristics for inducing chunks from parse trees. We observe that our maximal left-branching heuristic outperforms right branching by 20 points in phrase F1. Additionally, we introduce a thresholding approach that only extracts one-word and two-word chunks since most of the groundtruth chunks contain one or two words. The performance of such a heuristic is better than right-branching but worse than our left-branching approach. The findings support our hypothesis that right-branching is a common structure in English and does not suggest meaningful chunks. Conversely, left-branching identifies closely related words and is an effective heuristic for inducing chunks from parse trees.

\subsubsection{Ablation Study on the HRNN Architecture} \label{sec: Ablation Study on the HRNN Architecture}
Table~\ref{tab: HRNN ablation} presents an ablation study on the HRNN model. Here, we only consider pretraining HRNN from unsupervised parsers, while excluding the finetuning phase on downstream tasks. This allows us to eliminate distracting factors in this ablation study.

We experimented with various neural architectures for the chunker, which learns from Compound PCFG with the maximal left-branching heuristic. We see all variants outperform the Compound PCFG chunker (Lines 2--9 vs.~Line 1 in Table~\ref{tab: HRNN ablation}), demonstrating that machine learning models can effectively smooth out noise and mitigate the imperfection of chunk heuristics.

Among different neural models, we see our HRNN outperforms traditional 1-layer and 2-layer RNNs (Lines~2--4), suggesting that the HRNN's switching gate provides useful autoregressive information during the chunking process. Moreover, the underlying BERT model provides valuable contextual information and can largely boost the performance when the number of RNN layers is controlled (Lines~2--3 vs.~6--7). Overall, our HRNN with BERT yields the highest performance, justifying our architecture design.

\begin{figure}[!t]
\includegraphics[width=0.95\linewidth]{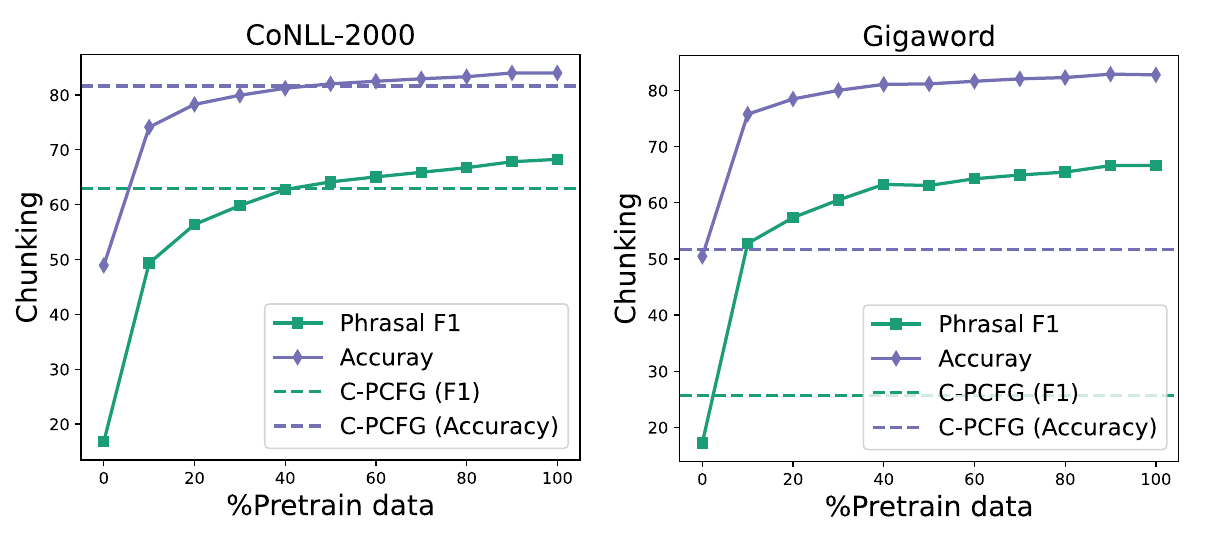}
\caption{Analysis of the size of the pretraining dataset.} 
\label{fig: ablate pertaining data curve}
\end{figure}

\subsubsection{Analysis of the Size of Pretraining Data}
We investigate the effect of the dataset size for pretraining our HRNN. We report the learning curves on two datasets: 1) CoNLL-2000, which is in the same domain as the pretraining data but with a different data split (test vs.~train), and 2) Gigaword, which is considered an out-of-domain corpus in the pretraining stage. We incrementally increased the percentage of pretraining data from 0\% to 100\%, and the resulting pretrained HRNN performance is shown in Figure~\ref{fig: ablate pertaining data curve}. They are compared with the chunking performance achieved by the Compound PCFG chunker (dashed lines), which is the source of our pretraining data.

For both CoNLL-2000 and Gigaword datasets, the initial increase in pretraining data size leads to a significant boost in the model's chunking performance, but the improvement slows down when the pretraining dataset is large. It is worth noting that the HRNN outperforms the Compound PCFG on Gigaword,  even with a small pretraining dataset (\textasciitilde 10\%). This highlights that, compared with heuristic chunk induction, our approach is able to improve the generalizability of unsupervised chunking.

\subsubsection{Analysis of Finetuning}
\label{sec: analysis of finetuning}
Chunking structure discovery with finetuning HRNN on downstream tasks requires several treatments: 1) We need to pretrain the HRNN, whose architecture has been analyzed in Section~\ref{sec: Ablation Study on the HRNN Architecture}, and 2) We need to apply an auxiliary loss push HRNN's predicted chunking probability away from $0.5$, as presented in Equation~\eqref{eqn: aux loss}. We analyze the effect of pretraining and the auxiliary loss in this part.

\begin{figure}[!t]
    \includegraphics[width=0.95\linewidth]{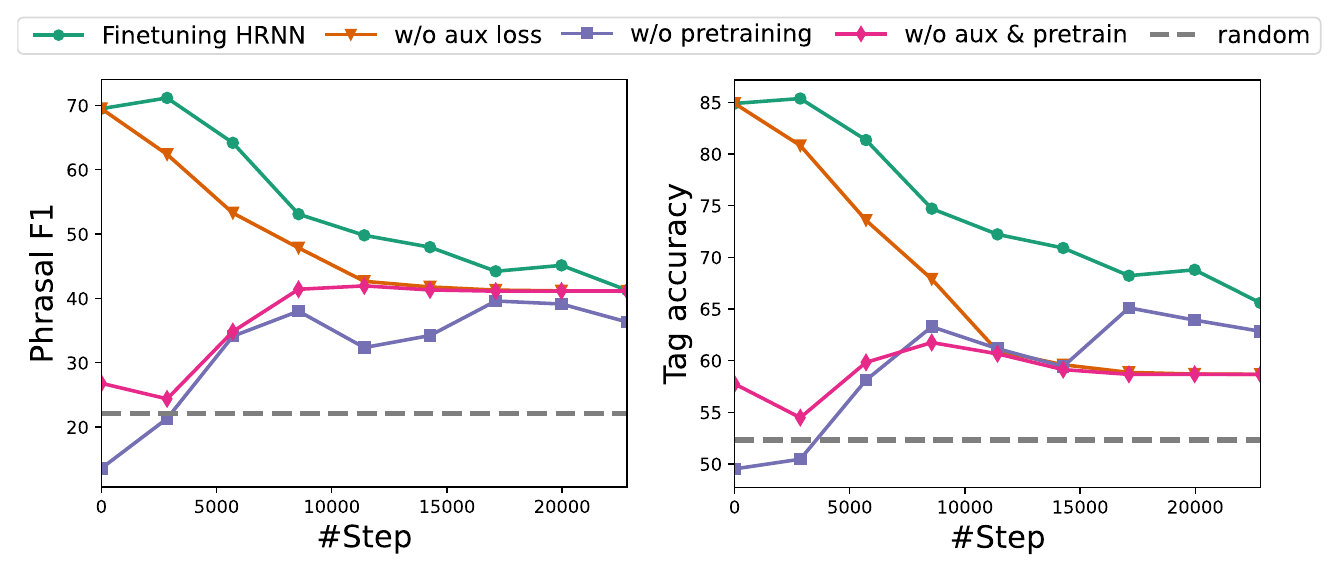}
    \caption{Ablation study of our HRNN finetuning method. We plot the learning curves in terms of phrasal F1 (top) and tag accuracy (bottom).}
    \label{fig: HRNN fine-tune ablation}
\end{figure}

Figure~\ref{fig: HRNN fine-tune ablation} presents the learning curve on the downstream dataset Gigward. As seen, the HRNN with both pretraining and auxiliary loss is the only configuration that achieves an improved chunking performance (green curves with dots), whereas other settings (without pretraining and/or without auxiliary loss) yield worse performance. This ablation study underscores the significance of both the pretraining strategy and auxiliary loss for HRNN's chunking structure discovery in downstream tasks. 

\begin{figure}[!t]
\includegraphics[width=0.95\textwidth]{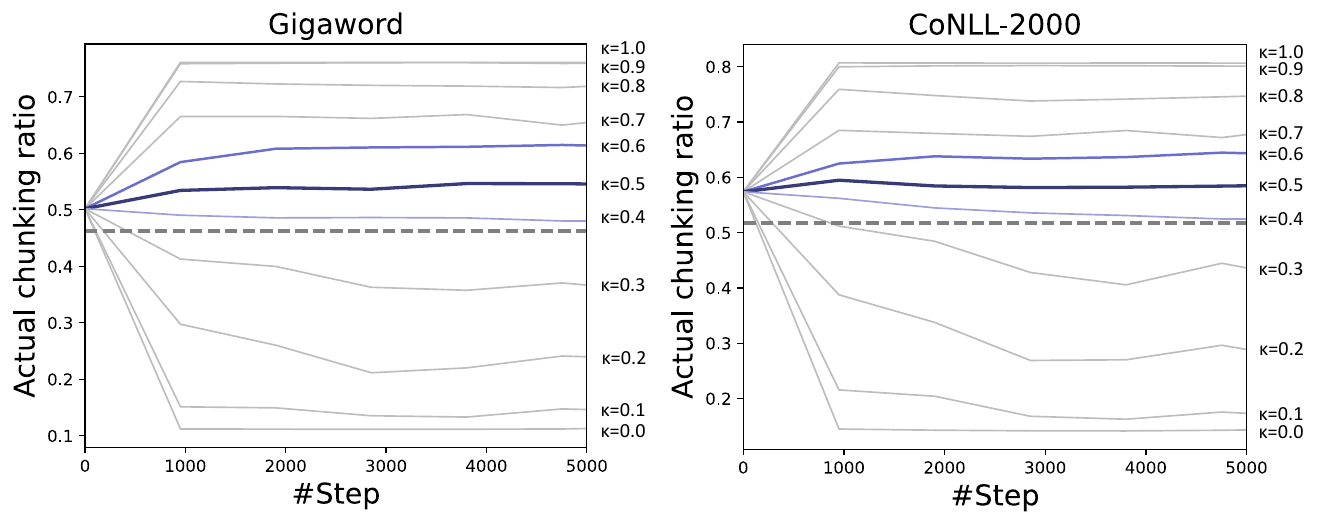}
\caption{This analysis examines the effect of the chunking strength hyperparameter $\kappa$ on the actual chunking ratio. The dashed line represents the groundtruth chunking ratio in the dataset. The purple lines indicate that the finetuned chunking performance outperforms the pretrained model. Additionally, the color depth and width of these lines indicate the ranking of the chunking performance achieved, with deeper and wider lines indicating better performance.}
\label{fig: cutrate}
\end{figure}

Moreover, the auxiliary loss provides the opportunity to control the chunking granularity. This can be shown by varying the chunking strength hyperparameter $\kappa$ in Equation~\eqref{eqn: aux loss} from $0$ to $1$ in increments of $0.1$. The results, presented in Figure~\ref{fig: cutrate}, confirm that $\kappa$ is indeed able to control the ratio of predicted cuts (i.e., chunking granularity), although the percentages may not match as there could be other factors contributing to the model's behavior. 
Interestingly, closely aligning the predicted and groundtruth chunking ratios does not lead to the best performance. The optimal $\kappa$ value turns out to be $0.5$ in our experiments, even though the predicted ratio slightly exceeds the groundtruth ratio.

We also notice an intriguing phenomenon in Figure~\ref{fig: HRNN fine-tune ablation}: all models converge to a phrase F1 score of approximately $40$. By manually checking their predictions, we realize that the models will eventually predict a cut for almost all steps, e.g., every word is a chunk by itself. Even with our full approach (green curves), the chunking performance is only improved within limited training steps, after which it drops as well. This raises a curious question about the dynamics of the chunking structure during the process of downstream-task finetuning.

\begin{figure}[!t]
\includegraphics[width=0.86\linewidth]{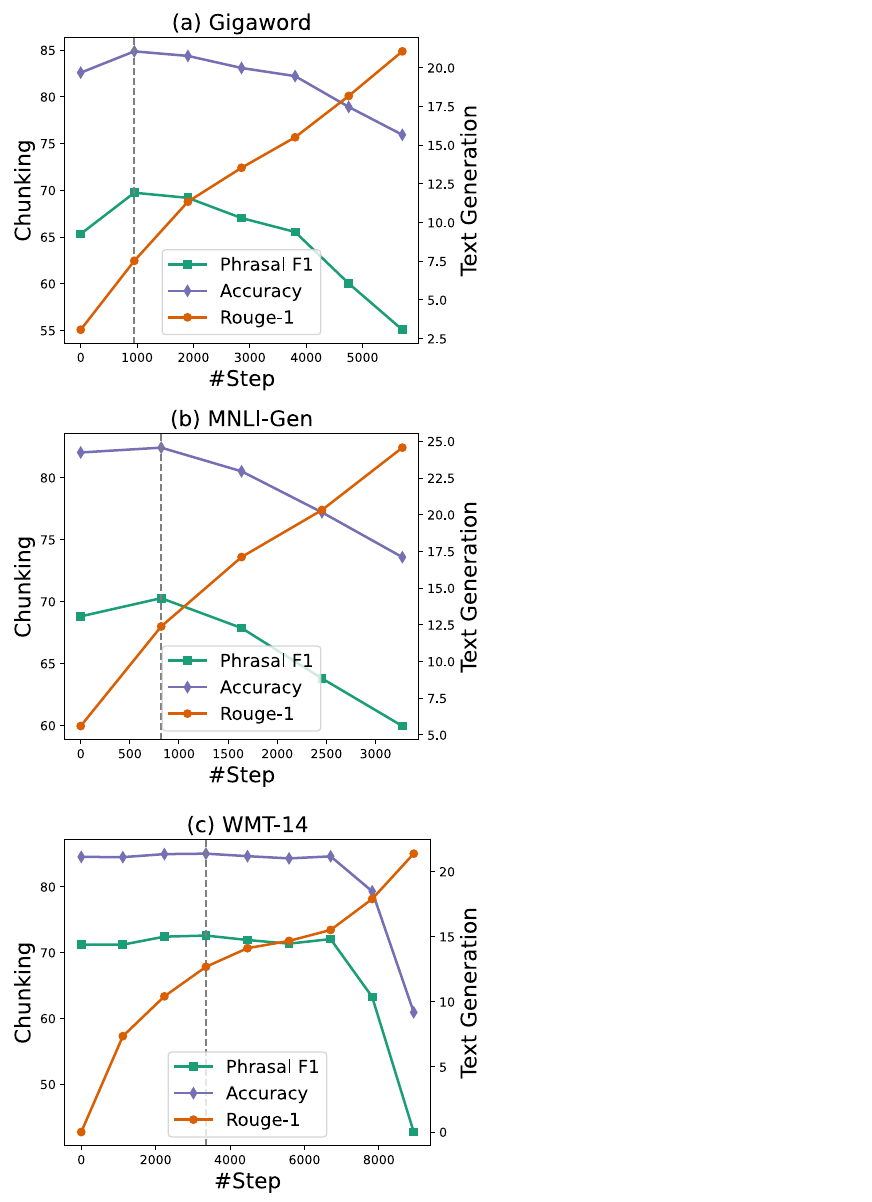}
\caption{Learning curves of finetuning HRNN with text generation tasks. Phrasal F1 and accuracy reflect the chunking performance, while the Rouge1 score measures the text generation performance. The gray dashed lines mark the steps when HRNN achieves the highest chunking performance.}
\label{fig:learning curve}
\end{figure}

To further investigate this, we show in Figure~\ref{fig:learning curve} the learning curves of both chunking performance and the downstream-task performance, evaluated by Rouge1 score~\cite{lin-2004-rouge}. The learning curves show that the emerging of meaningful chunking structures is a transient phenomenon only during the early stage of downstream-task learning. The chunking performance drops after several thousand iterations, while the downstream-task performance continues to increase monotonically.

We hypothesize that chunking information (or linguistic information in general) is a meaningful abstraction of human languages. When our exquisitely designed model has not mastered the downstream NLP task well (e.g., at an early stage), it will produce structures akin to linguistically defined chunks, which serve as a convenient intermediate step for the downstream task. As training proceeds, however, such linguistically meaningful structures are abandoned by the model, as deep neural networks are known to achieve high performance with end-to-end unexplainable features~\cite{wu2023weakly}. Our work sheds light on future research in linguistics.

\begin{figure}[!t]
\includegraphics[width=0.95\linewidth]{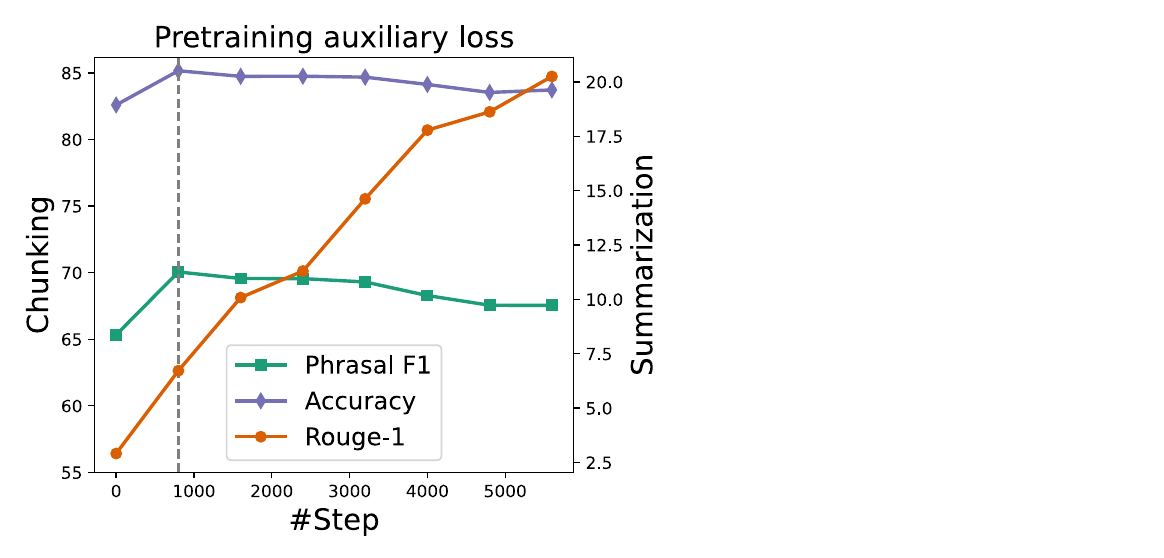}
\caption{Using our pretraining objective as the auxiliary loss during the finetuning of the downstream summarization task.}
\label{fig: ablate continual learning curve}
\end{figure}

Nevertheless, the drop of chunking performance can be easily alleviated by blending the pretraining and finetuning objectives. As seen in Figure~\ref{fig: ablate continual learning curve}, this treatment achieves similar peak chunking performance, which is mostly maintained during the finetuning stage. However, we do not adopt it as our main method. When we were developing our approach, our initial thought was that the pretrained chunking structures may not be perfect, so we refrained from adding the pretraining loss. Our finetuning allows the model to deviate from the pretraining structure, which unexpectedly leads to the intriguing phenomenon (i.e., the emerging and abandoning of chunking structures).

\begin{figure}[!t]
\includegraphics[width=0.95\linewidth]{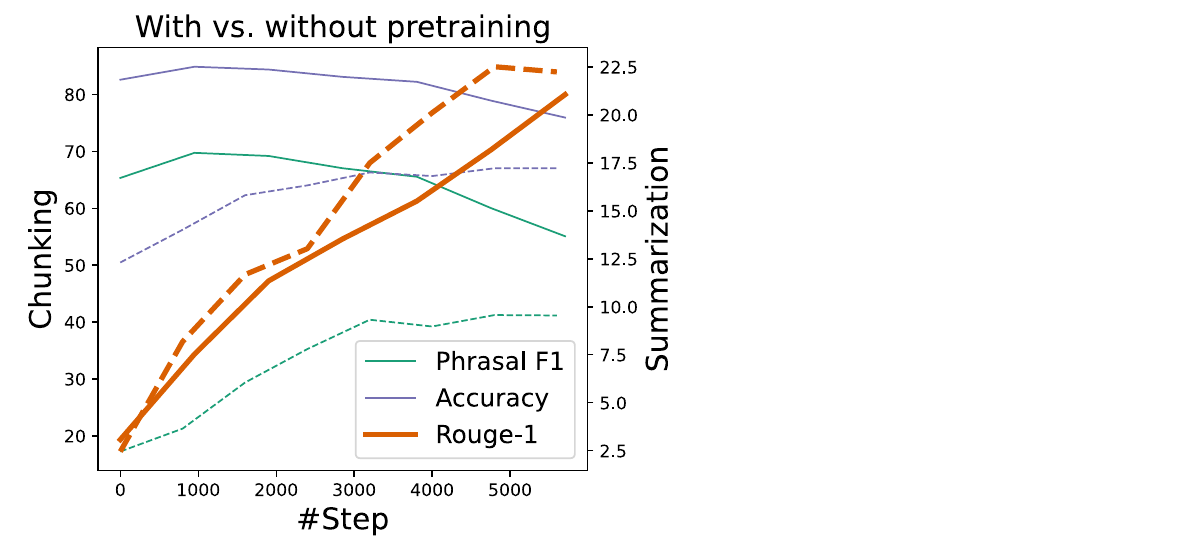}
\caption{Dashed lines represent the performance of finetuning a randomly initialized HRNN model on the downstream summarization task. Solid lines are from our full method. The chunking performance is evaluated with the in-domain source sentences.}
\label{fig: ablate pertaining learning curve}
\end{figure}

\subsubsection{Effect of Pretraining on the Downstream Task}
We study the effect of pretraining on the downstream text generation performance. We take summarization as the testbed and compare the pretrained HRNN with a randomly initialized one. Figure~\ref{fig: ablate pertaining learning curve} shows the learning curves for both chunking and summarizaation performance. Interestingly, we find that the pretrained HRNN demonstrates slightly slower training progress in the text generation task than a randomly initialized model. This is understandable because our pretrained HRNN serves as a linguistic regularizer, whereas a random initialization is more flexible. The phenomenon is consistent with the explanation in \citet{li-etal-2019-imitation}: there are multiple local optimal latent structures that can accomplish the downstream task, but with meaningful initialization, the downstream task may recover the linguistically plausible structures (such as chunking and parsing).

\begin{figure}[!t]
\includegraphics[width=0.95\textwidth]{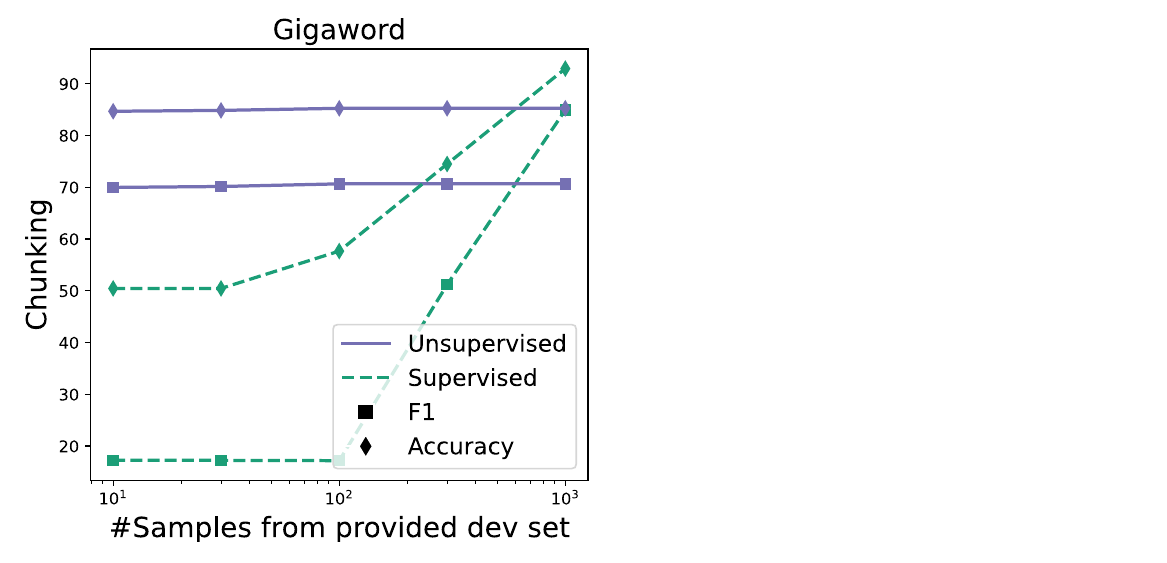}
\caption{The impact of validation set size on the performance of unsupervised and supervised HRNNs in the Gigaword chunking task.}
\label{fig: dev training}
\end{figure}

\subsubsection{Discussion on the Use of Labeled Validation Data} 
\label{section: Effect of the Size of Validation Data on Unsupervised Performance}
It is worth mentioning that we used labeled validation data in our main experiments. This is a common practice in previous unsupervised syntactic structure induction papers~\cite{shen2018ordered,drozdov-etal-2019-unsupervised-latent,shayegh2024ensemble}, with some authors claiming that such unsupervised methods are not ``fully unsupervised in the strictest sense of the term''~\cite{kim-etal-2019-compound}. Our work falls in same category. 

Certain researchers believe that using labeled validation sets is problematic for unsupervised grammar induction~\cite{shi-etal-2020-role}. We argue that not using a labeled validation set (e.g., using an unsupervised signal for validation) will eventually resort to the test set for hyperparamter tuning and model selection, because authors will have to repeatedly check their test performance to determine whether each of their validation strategies is effective or not. 

We nevertheless compare our unsupervised approach with using some labeled validation samples for training, by splitting the provided validation set into a training set and an actual validation set~\cite{shi-etal-2020-role}. In particular, we vary the number of labeled samples used, and take 70\% for training and 30\% for actual validation.

As shown in Figure~\ref{fig: dev training}, the performance of the unsupervised HRNN improves slightly when the validation set size increases from 10 to 100 samples, as the F1 score increases from 69.97 to 70.66 and  accuracy increases from 84.68 to 85.25; however, the improvement plateaus beyond this point. 

By contrast, the supervised HRNN does not show significant gains with an increase in training samples from 10 to 100, but its performance consistently improves as more data are added. Notably, while the supervised HRNN still underperforms the unsupervised model at 300 samples, its performance surpasses the unsupervised counterpart when the labeled dataset size reaches 1,000 samples.

These results indicate that our unsupervised method is robust to the validation set size, and works well with an extremely small validation set such as 10 samples. This also suggests that our approach could be useful in low-resource languages where large labeled datasets are not available. In such scenarios, even experts visually checking model outputs could be sufficient for effective model selection.

\subsubsection{Case Studies} 
\begin{figure}[!t]
\includegraphics[width=\textwidth]{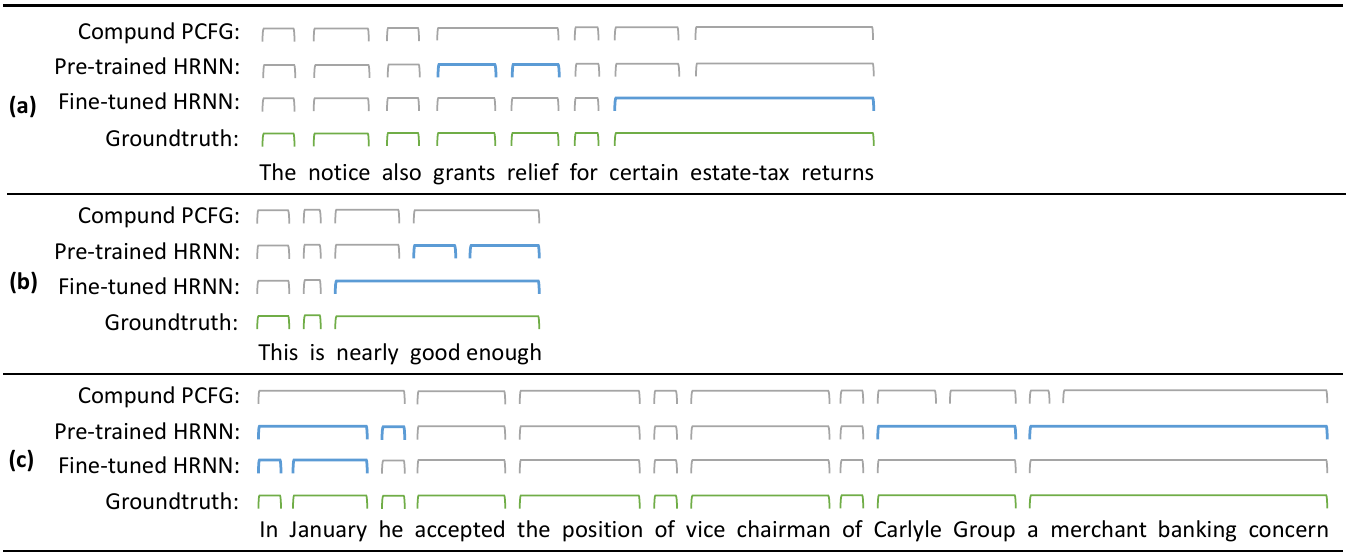}
\caption{Examples of chunking structures generated by Compound PCFG, pretrained HRNN, and finetuned HRNN are provided. Differences are highlighted in thick blue. Groundtruth chunks are also displayed in green for reference.}
\label{fig: case study}
\end{figure}

In Figure~\ref{fig: case study}, we show examples of chunking structures generated by Compound PCFG, pretrained HRNN, and finetuned HRNN in the summarization task. 

As seen, our pretrained HRNN produces longer noun phrases than Compound PCFG, such as \textit{Carlyle Group} and  \textit{a merchant banking concern} in Example (c), and finetuning HRNN can further detect longer noun phrases---such as \textit{certain estate-tax returns} in Example (a)---which are usually considered as a chunk by groundtruth labels.

Furthermore, our approach is able to correct nonsensical chunks produced by Compound PCFG. In Example (a), the two-word chunk \textit{grants relief} is split into \textit{grants} and \textit{relief} as they are not in the same semantic unit. In Example (c), the chunk \textit{In January he} is first split into \textit{In January} and \textit{he}, which makes more sense. Then the finetuned HRNN further splits \textit{In January} into \textit{In} and \textit{January}, which agree more with the groundtruth chunks.

In general, the pretrained HRNN not only effectively learns the chunking patterns from Compound PCFG, but also can smooth out its noise. The finetuned HRNN model further improves the pretrained model and achieves better performance of unsupervised chunking.

\section{Related Work}
This section briefly discusses recent advances in chunking and unsupervised grammar induction approaches, followed by recent literature on weakly supervised neuro-symbolic methods.  

\subsection{The Chunking Task}
Chunking is a prominent task in natural language processing with various applications, such as named entity recognition~\citep{wang2013joint}, syntactic analysis~\citep{zhou2012exploiting}, and fine-grained sentiment analysis~\citep{yang2013joint}. Chunking is also referred to as shallow parsing, where the goal is to separate a text into non-overlapping chunks that are syntactically connected~\citep{abney1991parsing}.

The CoNLL-2000 shared task~\citep{sang2000introduction}, derived from the Penn Treebank~\citep{marcus-etal-1993-building}, has provided a platform for system comparison of English chunking research. In the early stage of research, various classification models, such as support vector machines (SVMs; \citealp{kudo2001chunking}) and the Winnow algorithm~\citep{zhang2002text}, are applied to accomplish the chunking task. Additionally, sequential labeling models are extensively used for chunking; examples include the hidden Markov model (HMM; \citealp{molina2002shallow}) and conditional random fields (CRFs, \citealp{sha2003shallow, mcdonald2005flexible}). \citet{sun2008modeling} adopt latent-dynamic conditional random fields to the chunking task, which leads to further performance improvement.

\citet{zhou2012exploiting} propose to use chunk-level features for the chunking task, outperforming most previous research based on word-level features~\cite{kudo2001chunking, zhang2002text, molina2002shallow, mcdonald2005flexible, sun2008modeling}.
Inspired by this, we design a hierarchical RNN that captures chunk features that are formed dynamically.
We further leverage contextual word representations from a language model and enhance performance through weak supervision from downstream tasks.

\subsection{Unsupervised Syntactic Structure Induction}
Detecting syntactic structures without linguistic supervision has attracted significant interest due to its applications in scenarios with limited resources~\citep{clark-2001-unsupervised, klein2005unsupervised}.

Unsupervised parsing is one of the most widely explored areas in syntactic structure detection. It can be broadly divided into three areas: unsupervised constituency parsing~\citep{klein-manning-2002-generative, haghighi-klein-2006-prototype-driven, reichart-rappoport-2008-unsupervised, clark-2001-unsupervised}, which aims to derive a constituency parse tree for a given sentence; unsupervised dependency parsing~\citep{seginer-2007-fast, NIPS2001_89885ff2}, which focuses on identifying the grammatical relationships between words in a sentence; and unsupervised joint parsing~\citep{klein-manning-2004-corpus, shen2018neural}, which seeks to combine both constituency and dependency parsing into a unified framework.

Among these settings, unsupervised constituency parsing has been explored most. \citet{klein-manning-2002-generative} propose an approach that utilizes an expectation--maximization (EM) algorithm to model the constituency of each span. \citet{haghighi-klein-2006-prototype-driven} propose to employ a probabilistic context-free grammar (PCFG) based on manually designed features. \citet{kim-etal-2019-compound} introduce the Compound PCFG for unsupervised parsing. On the other hand, methods have been developed to obtain tagged parse trees without supervision.
\citet{reichart-rappoport-2008-unsupervised} perform clustering by syntactic features to obtain tagged parse trees. \citet{clark-2001-unsupervised} clusters sequences of tags based on their local mutual information to build tagged parse trees. 
Overall, these early studies mostly rely on heuristics, linguistic knowledge, and manually designed features.

Unsupervised parsing has revived interest in the deep learning era. \citet{socher-etal-2011-semi} propose a recursive autoencoder that constructs a binary tree by greedily minimizing the reconstruction loss. Other methods for learning recursive tree structures in an unsupervised manner include CYK-style marginalization~\citep{drozdov-etal-2019-unsupervised-latent} and Gumbel-softmax~\citep{Choi_Yoo_Lee_2018}.
\citet{yogatama2017learning} present a shift–reduce parser learned by reinforcement learning to improve performance on a downstream task such as textual entailment. However, evidence shows the above approaches do not yield linguistically plausible trees~\citep{williams-etal-2018-latent}. 

\citet{li-etal-2019-imitation} propose to transfer knowledge between multiple unsupervised parsers to improve their performance. Our pretraining method draws inspiration from such knowledge transfer, but we introduce insightful heuristics that generate chunk labels from unsupervised parsers. Additionally, we have developed a hierarchical recurrent neural network (HRNN) that learns from the chunk labels generated by our heuristics, and can further be finetuned with the downstream tasks for performance improvement.

Our proposed HRNN shares some similarities with the stack-LSTM~\citep{dyer-etal-2015-transition}, as both models aim to capture hierarchical structures of language. While the stack-LSTM can generate deep tree structures, our approach is restricted to two layers and is more suited to the chunking task.

Regarding unsupervised chunking, the speech processing community has addressed the task as a component of speech detection systems and utilized acoustic information to determine the chunks~\citep{pate-goldwater-2011-unsupervised, barrett-etal-2018-unsupervised}. Our work only relies on textual information and views unsupervised chunking as a new task of inducing syntactic structure.

\subsection{Weakly Supervised Neuro-Symbolic Methods}
In recent years, neuro-symbolic approaches have gained significant attention from the AI and NLP communities for interpreting deep learning models~\citep{liu2023neurosymbolic}. Many of these approaches treat neuro-symbolic methods as combining deep learning and symbolic representations. For example, \citet{lei-etal-2016-rationalizing} extract key phrases that rationalize the neural model's predictions on classification tasks.
\citet{10.5555/3304222.3304359} model text classification as a sequential decision process, and propose a policy network that is able to skip reading and make decisions. \citet{liang-etal-2017-neural} and \citet{pmlr-v70-mou17a} perform SQL-like execution based on input text for semantic parsing. 

However, training a neuro-symbolic model is more complicated than an end-to-end neural model, because the explainable discrete latent structure lacks ground truth for direct supervision. As a result, these approaches are typically trained in a weakly supervised manner, where the training signals only exist at the end of the entire model. Reinforcement learning or its relaxation, such as Gumbel-softmax~\citep{jang2017categorical}, is often used in the training procedure due to the non-differentiable latent structure. Recently, \citet{wu2023weakly} propose neural fuzzy logic to explain the predictions of the natural language inference model, which is trained end-to-end. 
Our HRNN model uses switching gates as the intermediary of the weak supervision from text generation for chunk induction, which allows it to be trained with different downstream datasets.

\section{Limitations and Future Work}

One limitation of this paper is that we have only applied our approach to the English language, as our HRNN chunker is pretrained by unsupervised parsers (which are mainly built for English) and our maximal left-branching heuristic for chunk induction utilizes the right-branching prior of English. To generalize unsupervised grammar induction to other languages, we may develop models that incorporate language-specific prior, e.g., left branching for Japanese~\cite{martin2003reference}. Another viable path is to translate a sentence into English, perform unsupervised syntactic structure prediction on the English translation, and map the structure back to the original language. Extending this, multilingual unsupervised grammar induction is also a promising research direction, since our results show that translation as a downstream task helps the discovery of syntactic structures.

We acknowledge that the current unsupervised chunking performance is still much lower than the supervised one, but with persistent research efforts, we expect it will be much improved in the near future. A recent study in our team addresses unsupervised parsing and has achieved close performance to supervised methods, especially for domain-shift scenarios~\cite{shayegh2024ensemble, shayegh2024ensemble1}.

Our long-term vision is to build an interpretable NLP system, trained in an end-to-end fashion, that can detect semantic units (chunks), determine their relationships, perform task-specific reasoning, and ultimately achieve high performance for the task of interest~\cite{liu2023neurosymbolic}. The main obstacle to this ambitious goal is the difficulty of end-to-end training: this paper performs chunking during text generation, but the obtained chunking structure is later abandoned by the model; another study of ours performs chunk detection and alignment by heuristics, after which a differentiable fuzzy logic mechanism is designed to perform logical reasoning~\cite{wu2023weakly}. In future work, we plan to put these components together and train our HRNN in the NLI reasoning task; since the granularity of chunks is usually ambiguous, training an end-to-end chunker in a downstream task may yield chunks that are more suited to the task. 
Our goal is to make traditional NLP pipelines~\cite{manning-etal-2014-stanford} end-to-end learnable. We anticipate that, with proper development, such a neuro-pipeline approach will achieve high performance in end tasks similar to black-box neural networks, while excelling in terms of explainablility and interpretability.

\section{Conclusion} 
In this paper, we propose a framework to induce chunks in an unsupervised manner as syntactic structure discovery. Specifically, we propose a hierarchical RNN (HRNN) with switching gates to learn from the chunk labels induced by a state-of-the-art unsupervised parser. The HRNN is further finetuned with various downstream text generation tasks to achieve better chunking results. 

The experiments show that our approach largely improves unsupervised chunking. Additionally, we provide comprehensive analyses of our HRNN, the pretraining stage using unsupervised parser-induced labels, as well as the finetuning stage in a downstream task. We also point out limitations and discuss future directions. 

\section*{Acknowledgments}
We thank the reviewers and editors for their efforts and valuable comments. The research is supported in part by the Natural Sciences and Engineering Research Council of Canada (NSERC) under Grant No.~RGPIN2020-04465, an Alberta Innovates Project, the Amii Fellow Program, the Canada CIFAR AI Chair Program, a UAHJIC project, a donation from DeepMind, and the Digital Research Alliance of Canada (alliancecan.ca).

\starttwocolumn
\bibliography{compling_style}

\end{document}